\documentclass[letterpaper,10pt,conference]{ieeeconf}

\usepackage[nocompress]{cite}
\usepackage{amsmath,amssymb}
\usepackage{graphicx}
\usepackage{balance}
\usepackage{dblfloatfix}
\usepackage[caption=false,font=footnotesize]{subfig}
\usepackage{diagbox}
\usepackage{comment}
\makeatletter
\let\labelindent\relax
\makeatother
\usepackage[shortlabels]{enumitem}

\usepackage[hidelinks]{hyperref}
\usepackage{booktabs}
\usepackage{siunitx}
\usepackage[most]{tcolorbox}

\usepackage[capitalize,nameinlink,noabbrev]{cleveref}

\usepackage{xcolor}
\definecolor{customblue}{RGB}{24,85,155}
\definecolor{customgreen}{RGB}{55,120,39}

\usepackage{etoolbox}
\newtoggle{ano}
\newtoggle{ext}

\togglefalse{ano}

\newtheorem{theorem}{Theorem}

\newtheorem{definition}{Definition}

\newtcolorbox{highlightbox}{
  colback=blue!10!gray!10,
  left=1pt,
  right=1pt,
  top=1pt,
  bottom=1pt,
  boxrule=1pt,
}

\IEEEoverridecommandlockouts

\title{\LARGE \bf
Learning Actuator-Aware Spectral Submanifolds \\ for Precise Control of Continuum Robots
}

\iftoggle{ano}{
\author{\textit{Authors List Omitted for Anonymized Review$^{1}$}%
\vspace{2em}
\thanks{$^{1}$ Formatting placeholder.
\vspace{3em}}
}
}{
\author{Paul Leonard Wolff$^{1,3}$, Hugo Buurmeijer$^{1}$, Luis Pabon$^{1}$, John Irvin Alora$^{1}$, Mark Leone$^{1}$, \\ Roshan S.~Kaundinya$^{2}$, Amirhossein Kazemipour$^{3}$,
Robert K. Katzschmann$^{3}$ and Marco Pavone$^{1,4}$%
\thanks{$^{1}$ Autonomous Systems Lab, Stanford University, USA. Emails: {\tt\small \{wolffp, hbuurmei, lpabon, jjalora, mleone, pavone\} @stanford.edu}.}%
\thanks{$^{2}$ Institute of Mechanical Systems, ETH~Zurich, Switzerland. Email: {\tt\small roshan.kaundinya@mavt.ethz.ch}.} %
\thanks{$^{3}$ Soft Robotics Lab, ETH~Zurich, Switzerland. Emails: {\tt\small \{wolffp, akazemi, rkk\}@ethz.ch}.}%
\thanks{$^{4}$ NVIDIA Research, Santa Clara, USA.}%
}
}

\begin{document}

\maketitle
\thispagestyle{empty}
\pagestyle{empty}

\begin{abstract}
Continuum robots exhibit high-dimensional, nonlinear dynamics which are often coupled with their actuation mechanism.
Spectral submanifold (SSM) reduction has emerged as a leading method for reducing high-dimensional nonlinear dynamical systems to low-dimensional invariant manifolds. 
Our proposed control-augmented SSMs (caSSMs) extend this methodology by explicitly incorporating control inputs into the state representation, enabling these models to capture nonlinear state-input couplings. 
Training these models relies solely on controlled decay trajectories of the actuator‑augmented state, thereby removing the additional actuation‑calibration step commonly needed by prior SSM‑for‑control methods.
We learn a compact caSSM model for a tendon-driven trunk robot, enabling real-time control and reducing open-loop prediction error by 40\% compared to existing methods.
In closed-loop experiments with model predictive control (MPC), caSSM reduces tracking error by 52\%, demonstrating improved performance against Koopman and SSM based MPC and practical deployability on hardware continuum robots.
\end{abstract}

\section{Introduction}

Continuum robots are constructed from compliant materials and actuated through distributed deformation rather than rigid joints.
These designs enable safe interaction in human-centered environments while expanding the design space for applications including manipulators, soft grippers, and morphing structures \cite{Trivedi2008,Katzschmann2018}.
However, the same compliance that yields these benefits complicates modeling and control: large deformations, geometric nonlinearities, and state-dependent couplings are common.
Furthermore, continuum-mechanics ODEs discretize into high-dimensional systems, making real-time implementation of optimization-based methods like MPC computationally prohibitive \cite{RaoWR98,Malyshev2018}.

Model order reduction (MOR) seeks low-dimensional coordinates that capture a system’s essential behavior while enabling fast inference. Classical linear-subspace methods, such as proper orthogonal decomposition (POD) and principal component analysis (PCA), perform well when system dynamics align with a fixed global basis, but struggle with strong nonlinearities. 
Recent advances in nonlinear, data-driven reduction on low-dimensional spectral submanifolds (SSMs) have established them as a leading MOR technique that preserves essential nonlinear dynamics while remaining computationally tractable for real-time control\cite{haller2016nonlinear,CenedeseAxasYangEritenHaller2022,alora2025discovering}.

Equally important alongside MOR is recognizing that, in compliant robots, the actuator channel itself is dynamic. Tendons, pneumatic lines, valves, and transmissions introduce delays and bandwidth limits, while friction, stiction, saturation, and hysteresis cause inputs to enter the system dynamics in nonlinear, state-dependent ways \cite{Do2017PerformanceControl}.
Ignoring these effects leads to underfitting of transients and degraded predictive performance, whereas jointly modeling actuator and system dynamics yields a model that better captures the dynamics most relevant for control.

Based on these observations, our goal is to learn reduced-order models for high-dimensional systems and their actuators with dynamics of the kind:
\begin{figure}[t!]
    \centering
    \includegraphics[width=\linewidth]{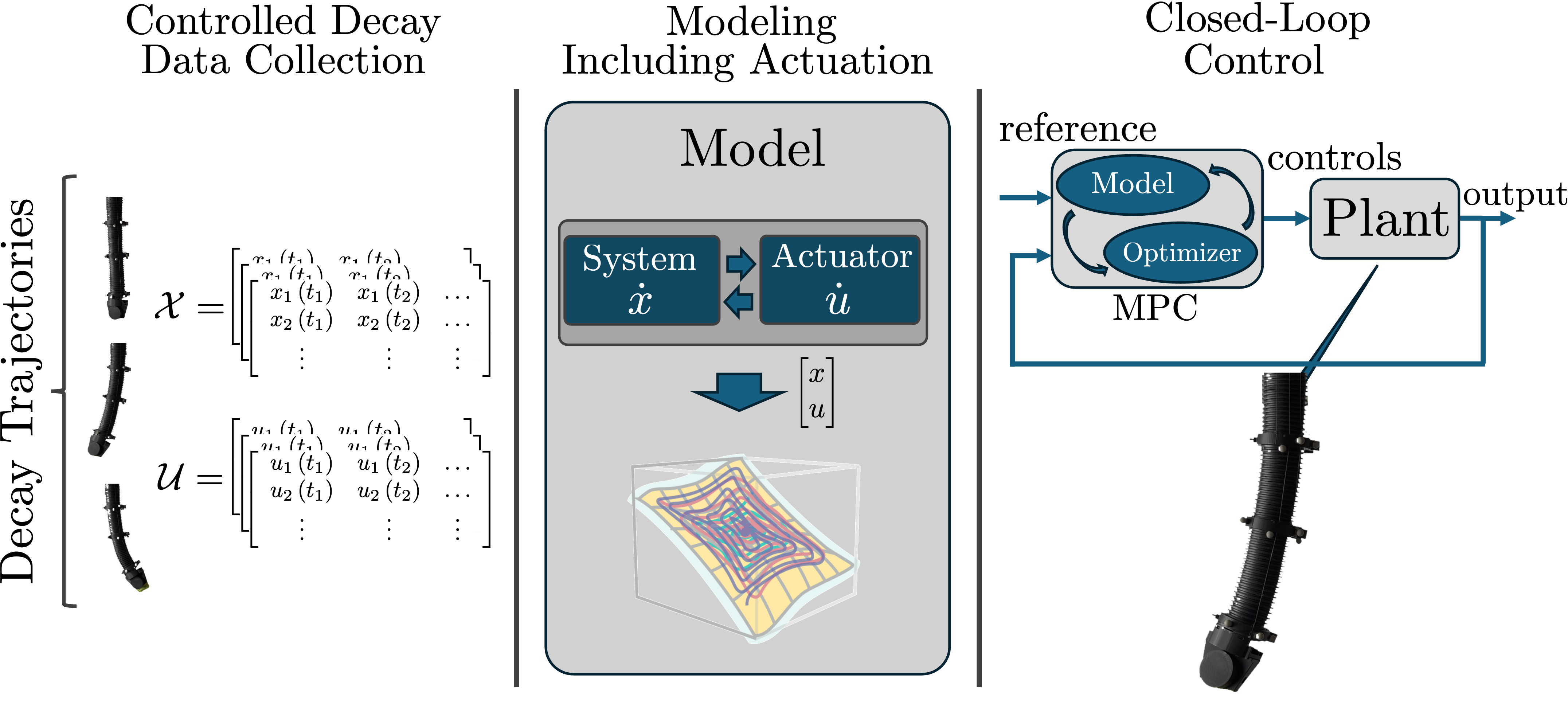}
    \caption{Workflow for deploying SSMs on hardware for closed-loop control: (a) fully automated decay-based data collection, (b) actuator-aware manifold modeling from data, (c) real-time MPC rollout on hardware.}
    \label{fig:title_figure}
    \vspace{-1.1em}
\end{figure}
\begin{equation}\label{eq:general_ode_intro}
    \begin{aligned}
        \dot{x}(t) &= f\big(x(t),u(t)\big),\\
        \dot{u}(t) &= \Lambda\big(u(t)-u_{\mathrm{ref}}(t)\big),
    \end{aligned}
\end{equation}
where $x(t)\in\mathbb{R}^{n_f}$ denotes the full-order state, 
$u(t)\in\mathbb{R}^{m}$ the control input applied through the actuators, and $f:\mathbb{R}^{n_f} \times \mathbb{R}^{m}\to\mathbb{R}^{n_f}$ the nonlinear dynamics. 
We model actuator dynamics as a first-order ODE with the matrix $\Lambda \in \mathbb{R}^{m\times m}$ capturing the dominant behavior of servo-driven tendon actuation. 
For pressure-driven fluidic actuators with higher-order or non-linear couplings, the state can be augmented with additional actuator-side states (e.g., pressure).  $u_{\mathrm{ref}}(t)$ denotes the reference input commanded to the actuators.

To tackle this problem, we introduce \textit{control-augmented Spectral Submanifolds }(caSSMs).
Our contributions are:
\begin{enumerate}[leftmargin=*]
  \item \textbf{Framework}: caSSM embeds actuator states within the full state space to yield reduced-order models that capture coupled system–actuator dynamics.  
  \item \textbf{Theory}: derivation of external control contribution and linear actuator dynamics within caSSM.
  \item \textbf{Pipeline}: an automated procedure that produces controller-ready models from decay trajectories (Fig.~\ref{fig:title_figure}).
  \item \textbf{Validation}: experiments on a tendon-driven trunk robot showing 40\% lower prediction error and 52\% higher tracking accuracy than state-of-the-art Koopman- and SSM-based MPC baselines.
\end{enumerate}

\noindent\textbf{Outline:} We survey related work in \Cref{sec:related_work} and provide relevant SSM preliminaries in \Cref{sec:background_ssm}.
Next, \Cref{chap:caSSM} introduces the caSSM formulation and derives the inclusion of actuator dynamics.
We detail the hardware data pipeline and learning procedure in \Cref{sec:deploy_cassm}.
Finally, we evaluate open-loop prediction and closed-loop MPC on a tendon-driven trunk robot in \Cref{sec:experiments} and conclude in \Cref{sec:conclusion_outlook}.


\section{Related Work}
\label{sec:related_work}

Modeling soft robots spans physics-based discretizations and data-driven surrogates. Finite element analysis (FEA) couples nonlinear hyperelasticity with time integration to handle complex geometries, contacts, and boundary conditions accurately. However, its state dimension and per-step solve costs are prohibitive for real-time, optimization-based control \cite{Duriez2013}.
MPC with horizon \(T\), \(n\) states, and \(m\) inputs yields dense quadratic programs scaling as \(\mathcal{O}\!\big(T^3(n{+}m)^3\big)\), reducible via block-tridiagonal structure to \(\mathcal{O}\!\big(T(n{+}m)^3\big)\) \cite{RaoWR98}.
This cubic dependence on \(n\) motivates MOR.

Projection-based MOR methods, such as proper orthogonal decomposition (POD), reduce dynamics onto linear subspaces and are optimal for linear systems, yet their accuracy quickly deteriorates under strong nonlinearities. Trajectory piecewise linearization (TPWL) extends applicability through multiple local linearizations but introduces residual modeling errors \cite{KunischVolkwein2002,RewienskiWhite2003,TonkensLorenzettiPavone2021}.
To overcome such limitations, snapshot-based, data-driven MOR instead extracts low-order structure directly from trajectories: dynamic mode decomposition (DMD) fits linear operators to observed dynamics \cite{Schmid2010}, while Koopman/Extended DMD (EDMD) lifts states into feature spaces to represent nonlinear dynamics linearly \cite{Williams2015,BruderRV19}. In both cases, performance hinges on the choice of observables or basis functions, where richer feature sets capture more dynamics but inflate dimensionality and hinder real-time MPC integration.



To capture nonlinearities beyond linear-operator approximations, machine learning approaches like Gaussian processes or deep networks learn input–deformation maps and embed ODE residuals for consistency. But they may require large curated datasets and performance can degrade under distribution shift, e.g., when extrapolating to unseen trajectories, amplitudes, or operating conditions \cite{GillespieBestTownsendWingateKillpack2018, YangDuriez2021}. Related hybrid models combine analytical models with learned corrections to improve data efficiency and extrapolation, but inherit modeling bias from the chosen physics prior \cite{Reinhart2017Hybrid}.

As an alternative to purely data-driven regression, SSMs provide a principled, dynamics-based reduction: they are smooth, low-dimensional invariant manifolds that capture the dominant attracting dynamics with few coordinates, generalizing slow eigenspaces to nonlinear settings. They can be identified directly from trajectory data using libraries like \texttt{SSMLearn} and \texttt{fastSSM} \cite{CenedeseAxasYangEritenHaller2022, haller2016nonlinear}.
However, while SSM-based approaches have also been extended to non-autonomous systems, including recent work demonstrating effective control after additional calibration, actuator effects and input dynamics are typically not modeled explicitly \cite{alora2025discovering,Kaundinya2025aSSM,buurmeijer2025taminghighdimensionaldynamicslearning}. Instead, control is often introduced through a separate calibration step. In contrast, our formulation incorporates actuator dynamics directly into the model from the outset. It thereby captures coupled actuation–system interactions while avoiding the additional experiments that post hoc calibration would require.

\section{Background on Spectral Submanifolds}
\label{sec:background_ssm}

We first consider Eq.~\eqref{eq:general_ode_intro} without external control.
Taking the Jacobian of the dynamics $f$ with respect to the full state at the origin gives \(A:=D_x f(0,0) \in\mathbb{R}^{n_f\times n_f}\), such that we can write the autonomous dynamics as
\begin{equation}\label{eq:autonomous_dyn}
    \dot x(t) = f_\mathrm{aut}(x(t)) := A x(t)+f_\mathrm{nl}(x(t)),
\end{equation}
where the nonlinear terms are grouped into $f_\mathrm{nl}$ and
$A$ is assumed to be Hurwitz.
We then define the associated slow linear subspace as
\[
  E := E_{j_1}\oplus \cdots \oplus E_{j_n},
\]
where each $E_{j_k}$ is the real eigenspace associated with the eigenvalue $\lambda_{j_k}$ of $A$.
The direct sum of eigenspaces is chosen so that
\[
    \min_{\lambda\in\Sigma_E}\text{Re}(\lambda)\;>\;\max_{\lambda\in\Sigma_{\text{out}}}\text{Re}(\lambda),
\]
with $\Sigma_E$ being the set of eigenvalues corresponding to $E$, and $\Sigma_{\text{out}}$ being the complementary set of faster eigenvalues.
From here, we define (adopted from \cite{haller2016nonlinear}): 
\begin{definition}[Autonomous SSM]
The \emph{autonomous SSM}, \(W(E)\), is the smoothest nonlinear continuation of \(E\) that remains invariant under the autonomous dynamics, with \(\dim W(E)=\dim(E)=:n\).
\end{definition}

For the dynamics in \eqref{eq:autonomous_dyn}, this invariance is formally expressed as
\begin{equation}\label{eq:ssm_invariance}
  x(0)\in W(E)\;\Longrightarrow\; x(t)\in W(E)\quad \forall ~t\in\mathbb{R}^+.
\end{equation}

Because trajectories remain confined to \(W(E)\), the manifold can be parametrized by reduced coordinates with corresponding mappings to and from the ambient state space:
\begin{align}
  z &= v(x), \label{eq:ssm_map_v}\\
  x &= w(z), \label{eq:ssm_map_w}
\end{align}
where \(z\in\mathbb{R}^n\) are reduced coordinates and \(x\in\mathbb{R}^{n_f}\) are the full coordinates, with invertibility on the manifold:
\[
  x=(w\circ v)(x),\qquad z=(v\circ w)(z).
\]
SSMs thus provide nonlinear, low-dimensional coordinates in which the dominant system behavior can be modeled efficiently.

To extend this approach to controlled systems, most SSM-based models assume control-affine dynamics \cite{alora2025discovering,Kaundinya2025aSSM,buurmeijer2025taminghighdimensionaldynamicslearning}
\begin{equation}\label{eq:full_regular_system}
  \dot x(t) = f_\mathrm{aut}(x(t)) + B(x(t))u(t),
\end{equation}
which we will relax by focusing on the generic form of \eqref{eq:general_ode_intro}.

\subsection{Fitting SSMs from Observational Data}
For high-dimensional systems, such as soft robots, one generally cannot observe their full state.
Therefore, we instead construct the SSM in the observable space of measurements, which can be made sufficiently large by adding time-delay embeddings. 
The fitting of SSMs from observational data generally involves three steps: 1) fitting the chart \eqref{eq:ssm_map_v} and parameterization maps \eqref{eq:ssm_map_w}, 2) finding the on-manifold dynamics, and 3) calibrating the dynamics to include the effect of control inputs.
The first two steps use data obtained by launching the robot from various initial conditions and recording its autonomous decay trajectories.
The slow spectral subspace, $E$, can directly be found from these observations by performing a PCA and selecting the $n$ most dominant components.
This defines the chart map as
\begin{equation}\label{eq:oSSM_v}
    z = v(x_o) := V^\top x_o ,
\end{equation}
where $x_o \in \mathbb{R}^o$ is the observed state vector and the columns of $V \in \mathbb{R}^{o \times n}$ span $E$.
Assuming analyticity of $f_{\mathrm{nl}}$, $W(E)$ admits a convergent Taylor expansion over $E$, motivating a finite-order polynomial feature map for the parameterization \cite{haller2016nonlinear}, i.e.,
\begin{equation}\label{eq:oSSM_w}
    x_o = w(z) := V z + W_\text{nl} z^{2:n_w} ,
\end{equation}
where $z^{2:n_w}$ collects all monomials of $z$ from order $2$ up to $n_w$, and the coefficient matrix $W_\text{nl}$ is obtained via polynomial regression.
The reduced dynamics similarly is
\begin{equation}\label{eq:oSSM_r}
    \dot{z} = R_0 z + R_\text{nl} z^{2:n_r} + B_r u ,
\end{equation}
where $z^{2:n_r}$ denotes monomials of order $2$ through $n_r$, $R_0$ and $R_\text{nl}$ are fitted on autonomous data via polynomial regression, and $B_r$ on controlled data via least-squares (see \cite{alora2025discovering}).
We will refer to these resulting models as the \textit{origin SSMs} (oSSMs) to differentiate from the proposed approach.


\subsection{Assumptions and Practical Limitations}

SSM-based reductions adopt the following modeling assumptions:
(i) manifold-shaping dynamics arise solely from the system’s intrinsic behavior, neglecting actuator interactions \cite{haller2016nonlinear};
(ii) actuator dynamics are much faster than the system, enabling near-instant setpoint tracking \cite{alora2025discovering};
(iii) the control influence on dynamics is affine (linear input contribution) as in \eqref{eq:full_regular_system}.


In robotic systems, these assumptions are often violated \cite{Do2017PerformanceControl,Ma2022ElectricallyDrivenSoftActuators}. In particular, slow actuator dynamics can be non-negligible and therefore affect the choice of modal basis used for the reduced-coordinate representation. To assess when this matters, Fig.~\ref{fig:eigenvalue_distribution} gives a practical bandwidth diagnostic through the relative location of actuator modes (blue) and robot modes (green): when the actuator modes are well separated on the fast side, the actuators can be approximated as ideal. If the spectra overlap, actuator dynamics should be modeled explicitly. When actuator modes are slower than the robot modes of interest, control authority becomes insufficient for reliable tracking. Consequently, omitting actuator modes can yield reduced coordinates that miss relevant physical behavior and lead to significant model mismatch.

These limitations also affect data collection.
Slow actuators may be incapable of providing step inputs to generate decay trajectories that trace out the manifold, necessitating manual displacements to create nonzero initial conditions followed by instant release, thereby limiting practicality.
Finally, restricting control effects to affine terms impedes capturing nonlinear actuation characteristics (e.g., friction) common in tendon-driven soft robots.

\begin{figure}[b]
    \centering
    \includegraphics[width=0.9\linewidth]{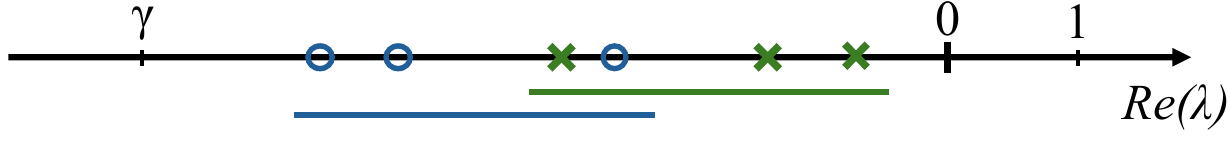}
    \caption{Eigenvalue layout: slow system modes are shown in \textcolor{customgreen}{green} and actuator modes in \textcolor{customblue}{blue}. The region left of \(\gamma \ll 0\) indicates rapidly decaying modes that can safely be neglected. Spectral overlap of modes signals that neglecting actuator dynamics can induce significant modeling errors.}
    \label{fig:eigenvalue_distribution}
\end{figure}

\section{Control-Augmented SSMs}
\label{chap:caSSM}

We introduce \emph{control-augmented SSMs} (caSSMs), a reduced modeling framework that explicitly embeds actuator dynamics into the SSM formulation.

To capture actuator behavior, we augment the state with the control input:
\[
  \tilde{x} \;=\; \begin{bmatrix} x \\ u \end{bmatrix} \in \mathbb{R}^{n_f+m}.
\]
By linearizing the state dynamics from \eqref{eq:general_ode_intro} around the origin, we obtain the (open-loop) system–actuator model
\begin{equation}
\begin{gathered}
  \dot{\tilde{x}}
  = 
  \begin{bmatrix} \dot{x} \\ \dot{u} \end{bmatrix} 
  =
  \underbrace{\begin{bmatrix} A & A_u \\ 0 & \Lambda \end{bmatrix}}_{\text{linear}}
  \begin{bmatrix} x \\ u \end{bmatrix}
  + 
  \underbrace{\begin{bmatrix} g(x,u) \\ 0 \end{bmatrix}}_{\text{nonlinear}}
  -
  \underbrace{\begin{bmatrix} 0 \\ \Lambda\,u_{\mathrm{ref}}(t) \end{bmatrix}}_{\text{setpoint drive}},\\[0.9ex]
\end{gathered}
\label{eq:augmented_general}
\end{equation}
with $\Lambda$ being Hurwitz. We thereby replace the control-affine assumption by a general coupling via a linear cross term $A_u \in \mathbb{R}^{n_f\times m}$ and a nonlinear term $g:\mathbb{R}^{n_f} \times \mathbb{R}^{m}\to\mathbb{R}^{n_f}$, thus allowing nonlinear input dependence. When the coupling mechanism (e.g. tendons) introduces identifiable slow modes on time scales relevant to the robot dynamics, the corresponding internal states should be explicitly included in $x$. $u_{\mathrm{ref}}(t)$ encodes the desired control input setpoints, which we prescribe and optimize over, while $u$ is the actual input applied to the system.
The Hurwitz condition on \(\Lambda\) enforces dissipative actuator dynamics whose eigenvalues set an actuation rate limit. The spectrum of the linearized augmented system is the union of those of \(A\) and \(\Lambda\).

To describe the caSSM, we use a graph-style parameterization over reduced coordinates \(z\in\mathbb{R}^n\): \begin{align} z &= V^\top \tilde{x}, \qquad V^\top V = I, \label{eq:graph_z}\\ \tilde{x} &= w(z) \;=\; V z + w_{\mathrm{nl}}(z), \label{eq:graph_y}\\ \dot z &= r(z,t) \;=\; R_0 z + r_{\mathrm{nl}}(z) + r_{\mathrm{ref}}(t) \label{eq:red_dyn} \end{align}
where $r_\text{ref}(t)$ captures the effect of external control inputs to the reduced system. Similar to \cite{Haller2025}, we choose the SSM parameterization specified in \eqref{eq:graph_y} to be independent of time.
Moreover, an orthogonal projection is chosen in accordance with \cite{alora2025discovering}, while recent work has explored generalizing to oblique projections \cite{buurmeijer2025taminghighdimensionaldynamicslearning,Bettini2025ObliqueProjection}.

\subsection{Derivation of External Control Contribution}
\label{sec:derive_ref_term}

We derive how the reference inputs appear on the SSM by identifying the reduced input term \(r_{\mathrm{ref}}(t)\) that preserves manifold invariance under the augmented dynamics \eqref{eq:augmented_general}. Throughout, we adopt the SSM assumption that the reference input remains small and bounded, \(|u_{\mathrm{ref}}(t)| < \delta\), such that the manifold is approximately preserved under external forcing.

\begin{theorem}
\label{thm:r_ref}
Under the augmented dynamics \eqref{eq:augmented_general} and the graph parameterization \eqref{eq:graph_z}–\eqref{eq:red_dyn}, the reduced input term that preserves manifold invariance is
\begin{equation}
r_{\mathrm{ref}}(t)
\;=\;
V^\top
\begin{bmatrix}
0\\[2pt]-\,\Lambda\,u_{\mathrm{ref}}(t)
\end{bmatrix}.
\label{eq:r_ref}
\end{equation}
\end{theorem}

\begin{proof}
On the manifold, \(\tilde{x}=w(z)\) by \eqref{eq:graph_y}. Differentiating and substituting \eqref{eq:red_dyn} gives
\begin{equation}
\label{eq:ref_chain}
\begin{aligned}
\dot{\tilde{x}}  &= V \dot{z} + D_z w_{\mathrm{nl}}(z)\dot z\\
      &= \bigl[V + D_z w_{\mathrm{nl}}(z)\bigr]
         \bigl(R_0z + r_{\mathrm{nl}}(z) + r_{\mathrm{ref}}(t)\bigr).
\end{aligned}
\end{equation}
The same derivative follows from the full augmented dynamics \eqref{eq:augmented_general} while plugging in \eqref{eq:graph_y}, giving
\begin{equation}
\label{eq:ref_ambient}
\begin{aligned}
\dot{\tilde{x}}
&= \begin{bmatrix}A & A_u\\[2pt] 0 & \Lambda\end{bmatrix}
   \bigl(Vz + w_{\mathrm{nl}}(z)\bigr) \\
&\quad + \begin{bmatrix}g(x,u)\\[2pt]0\end{bmatrix}
\;-\; \begin{bmatrix}0\\[2pt]\Lambda\,u_{\mathrm{ref}}(t)\end{bmatrix}.
\end{aligned}
\end{equation}
Equating \eqref{eq:ref_chain} and \eqref{eq:ref_ambient} and isolating the time-dependent terms, we obtain
\begin{equation}
D_z w(z)\,r_{\mathrm{ref}}(t)
\;=\;
\begin{bmatrix}0\\[2pt]-\,\Lambda\,u_{\mathrm{ref}}(t)\end{bmatrix}.
\label{eq:ref_collect}
\end{equation}

At this point, recall that the manifold parameterization is smooth and can be written as $w(z)=Vz+w_{\mathrm{nl}}(z)$ \eqref{eq:graph_y}. Its Jacobian therefore has the form
\[
D_z w(z) = V + D_z w_{\mathrm{nl}}(z),
\]
where $D_z w_{\mathrm{nl}}(0)=0$ and $\|D_z w_{\mathrm{nl}}(z)\|=\mathcal{O}(\|z\|)$. \footnote{Tangency gives $D_z w(0)=V$ hence $D_z w_{\mathrm{nl}}(0)=0$; analyticity yields $w_{\mathrm{nl}}(z)=\mathcal{O}(\|z\|^2)$ and thus $D_z w_{\mathrm{nl}}(z)=\mathcal{O}(\|z\|)$ \cite{haller2016nonlinear}.}

Close to the equilibrium, the nonlinear contribution is higher order in \(z\) such that we approximate
\begin{equation}
V\,r_{\mathrm{ref}}(t)
\;=\;
\begin{bmatrix}0\\[2pt]-\,\Lambda\,u_{\mathrm{ref}}(t)\end{bmatrix}.
\label{eq:ref_collect_approx}
\end{equation}
Left-multiplying by \(V^\top\) and using \eqref{eq:graph_z} yields \eqref{eq:r_ref}.
\end{proof}
This result is consistent with a more general statement found in \cite{Haller2025}, Section 4.2.

\subsection{Obtaining Linear Actuator Dynamics}
\label{sec:lambda_ident}

Inspecting the derived term \eqref{eq:r_ref} shows its dependency on the $\Lambda$-matrix, the linear actuator dynamics. However, this is an unknown quantity and shall be identified.
To do so, we examine the linear terms in $z$ when equating \eqref{eq:ref_chain} with \eqref{eq:ref_ambient}, yielding the linear invariance condition
\begin{equation}
  V R_0 \;=\; \begin{bmatrix} A & A_u \\ 0 & \Lambda \end{bmatrix} V .
  \label{eq:inv_linear}
\end{equation}

Practically, the matrix \(R_0 \in \mathbb{R}^{n \times n}\) is approximated by taking the Jacobian of the identified reduced dynamics, with the caSSM dimension $n$ chosen to include the modes associated with actuation, as illustrated in Fig.~\ref{fig:eigenvalue_distribution}.

Let $E_u \in \mathbb{R}^{m\times(n_f+m)}$ be defined as $E_u = \bigl[\,0_{m\times n_f}\;\; I_m\,\bigr]$,
which extracts the actuator components.
We can decompose $V = \begin{bmatrix} V_x^\top , V_u^\top \end{bmatrix}^\top \in \mathbb{R}^{(n_f+m)\times n}$, such that left-multiplying \eqref{eq:inv_linear} with $E_u$ gives
\begin{equation*}
    \begin{aligned}
        E_u V R_0 &= E_u \begin{bmatrix} A & A_u \\ 0 & \Lambda \end{bmatrix} V ,\\
        V_u R_0 &= \Lambda V_u .
    \end{aligned}    
\end{equation*}

Assuming $\mathrm{rank}(V_u)=m$, due to appropriate tuning, the linear actuator dynamics follows from least squares as
\begin{equation}
\label{eq:Lambda_hat}
    \Lambda \;=\; (V_u R_0)\,V_u^{\dagger} \;=\; (E_u V\,R_0)\,(E_u V)^{\dagger} ,
\end{equation}
where $(\cdot)^\dagger$ denotes the Moore–Penrose pseudoinverse.
In essence, \(V\) captures dominant robot-actuator motion patterns. 
Projecting them onto the actuator coordinates and fitting the best linear map yields the actuator dynamics matrix \(\Lambda\) consistent with manifold invariance.

Alternatively, future work could treat \(\Lambda\) as a design parameter with additional degrees of freedom:
(i) placing real eigenvalues near measured actuator rates for realism; 
(ii) choosing \(\text{Re}(\lambda(\Lambda))\) slightly more negative than dominant system modes to separate time scales and aid identification; 
(iii) introducing mild oscillations via complex-conjugate pairs to excite weakly damped modes;
(iv) cautiously aligning actuator eigenfrequencies with systems resonances to amplify informative responses.
This design choice trades realism, numerical conditioning, and excitation richness for identification and control.

\subsection{Augmenting with Actuator-Side Feedback}

Besides data-driven identification, we can incorporate state feedback into the actuator dynamics via the linear gain $H$ or the nonlinear term $h(x,u)$. For SSM existence in Section \ref{chap:caSSM}, we considered the stable case ($A$ Hurwitz).
Now, we analyze actuator-assisted stabilization when the system's linearization is not Hurwitz, denoted by $A^{\mathrm{u}}$, leading to the augmented state dynamics
\begin{equation}
\label{eq:augmented_feedback}
\begin{gathered}
\dot{\tilde{x}} =
\begin{bmatrix} A^\mathrm{u} & A_u \\ H & \Lambda \end{bmatrix} \tilde{x}
+ \begin{bmatrix} g(x,u) \\ h(x,u) \end{bmatrix}
- \begin{bmatrix} 0 \\ \Lambda\,u_{\mathrm{ref}}(t) \end{bmatrix} .
\end{gathered}
\end{equation}

Linearizing the actuator channel by setting $h(x,u)\equiv 0$ and $u_{\mathrm{ref}}\equiv 0$ shows how the feedback acts:
\[
\begin{gathered}
\dot u = Hx + \Lambda u 
\Rightarrow\\[0.3em]
U(s) = (sI-\Lambda)^{-1} H\,X(s) = K(s)\,X(s).
\end{gathered}
\]
This yields a dynamic feedback filter with low-pass behavior, where $K(s)\approx -\Lambda^{-1}H$ for $s\approx 0$ and $K(s)\to 0$ as $|s|\to\infty$. 
Hence, the feedback strongly influences slow modes while attenuating high-frequency content, in line with the actuator bandwidth.

Let
\[
  \tilde{A} \;:=\; \begin{bmatrix} A^\mathrm{u} & A_u \\[2pt] H & \Lambda \end{bmatrix}.
\]
be the linearized augmented system around an unstable equilibrium of $A^{\mathrm{u}}$.
If $(A^{\mathrm{u}},A_u)$ is stabilizable, $\Lambda$ is Hurwitz, and the actuator influences the unstable modes of $A^{\mathrm{u}}$, then there exists $H$ such that $\tilde{A}$ is Hurwitz.
Practically, this requires actuator bandwidth to cover these modes: the eigenvalues of $\Lambda$ must lie sufficiently left of the unstable or slow modes so that $K(s)\approx -\Lambda^{-1}H$ holds over the frequency range of interest. Otherwise, control authority is insufficient. Appendix~\ref{app:stability} gives a sufficient condition and proof.
\section{Deploying caSSMs on Hardware}
\label{sec:deploy_cassm}

This section details the pipeline to identify and employ caSSM on hardware: (i) the data collection procedure; (ii) feature-map selection; (iii) control reference injection; and (iv) the reduced-order MPC formulation.

\subsection{Data Collection}
\label{sec:data}
To capture the dominant nonlinear behavior and system–actuator coupling, we collect trajectories that (a) cover a broad operating range and (b) excite transverse modes through orthogonal perturbations that generate oscillatory transients. Such experiments increase signal variance, reveal mode interactions, and make actuator-driven effects more identifiable. 
Raw measurements are denoised offline using a Kalman filter with Rauch–Tung–Striebel (RTS) smoothing, which stabilizes learning without inducing phase shifts. During online control, we use raw measurements, as non-causal filtering would reintroduce phase shifts and the controllers proved robust enough to handle the measurement noise.

Since the full state of a continuum robot cannot be measured, we restrict the model to measurable observables. We denote the observable augmented state by $\tilde{x}_o = \begin{bmatrix} x_o^\top , u^\top \end{bmatrix}^\top \in \mathbb{R}^{L(o+m)}$, 
which may include up to $L-1$ delay embeddings of the observed control-augmented state. For an appropriate choice of $L$ or $o$, we refer to \cite{alora2025discovering,CenedeseAxasYangEritenHaller2022}.

\subsection[Feature Maps for w\_nl and r\_nl]{Feature Maps for \(w_{\mathrm{nl}}\) and \(r_{\mathrm{nl}}\)}
The oSSM models use polynomial feature maps as seen in \eqref{eq:oSSM_w} and \eqref{eq:oSSM_r}.
However, with finite data and potential overfitting, this can cause divergence in certain data regimes.
Similar to \cite{kaszas2025jointreducedmodellaminar}, who used radial basis functions to train nonlinear SSM feature maps for capturing chaotic attractors, we revisited this design choice and evaluated three families of feature maps:
(1) low-degree polynomial monomials, (2) random Fourier features (RFFs) approximating an RBF kernel as described in \cite{rahimi2007random}, and (3) small neural networks.
In our ablation study, RFF-inspired cosine features achieved the highest the highest in-domain accuracy and the least out-of-domain divergence. Results and implementation details are provided in Appendix~\ref{app:feature_maps}.
Therefore, we use the cosine feature map motivated by the random Fourier features (RBF approximation) for all fitted nonlinear mappings.

\subsection{Control Reference Vector}

As for the observed augmented state, the control reference vector as defined in \eqref{eq:r_ref} shall be embedded.
For embedding depth \(L\), the controller maintains a buffer of the most recent inputs. The contribution per embedding is
\[
\tilde{J} =
\begin{bmatrix}
0_{o \times o} & 0_{o \times m} \\
0_{m \times o} & -\Lambda
\end{bmatrix},
\quad
\tilde{J}_L = I_L \otimes \tilde{J}
\;\in\; \mathbb{R}^{\,(L(o+m))^2},
\]
which places zeros on the system rows and \(-\Lambda u_{\mathrm{ref}}(\cdot)\) on the actuator rows.

With this definition, the control reference vector in reduced coordinates is
\begin{equation}
\label{eq:control_ref_vector_compact}
r_{\mathrm{ref}}(t)
= V^\top \, \tilde{J}_L \, \tilde{x}_o(t).
\end{equation}

At runtime, the controller optimizes the embedded \(u_{\mathrm{ref}}\), causing the MPC complexity to grow linearly with the number of delay embeddings.



As an illustration, for embedding depth \(L=2\) the expression becomes
\begin{align}
r_{\text{ref}}
&= 
V^\top
\begin{bmatrix}
  0_{o} \\
  -\,\Lambda\,u_{\mathrm{ref}}(t) \\
  0_{o} \\
  -\,\Lambda\,u_{\mathrm{ref}}(t - dt)
\end{bmatrix},
\label{eq:control_ref_vector_exact}
\end{align}
showing how the current input and its one-step delay appear in the embedded formulation.

For faster solves, we approximate the input history by repeating the current input across all lags. 
With the same block structure \(\tilde{J}_L\), the control reference vector becomes
\begin{equation}
\label{eq:control_ref_vector_approx}
r_{\mathrm{ref}}(t) \;\approx\; 
V^\top \, \tilde{J}_L \, (1_L \otimes u_{\mathrm{ref}}(t)).
\end{equation}
This keeps the number of decision variables independent of \(L\) (only the current input is optimized), 
yielding much faster MPC solves with only minor accuracy loss. 
For example, with \(L=2\) the approximation expands to
\begin{align}
r_{\text{ref, approx}}
&= 
V^\top
\begin{bmatrix}
  0_{o} \\
  -\,\Lambda\,u_{\mathrm{ref}}(t) \\
  0_{o} \\
  -\,\Lambda\,u_{\mathrm{ref}}(t)
\end{bmatrix}.
\label{eq:control_ref_vector_approx_expanded}
\end{align}

\subsection{Low-Dimensional MPC Formulation}
\label{sec:low_dim_mpc}

For closed-loop control on hardware, GuSTO-based MPC~\cite{gusto_mpc} is used and implemented in JAX to allow for just-in-time compilation.
The exact low-dimensional MPC formulation is given by
\begin{equation}
\hspace{-0.5em}
\begin{aligned}
  \min_{u_\text{ref}(\cdot)}\quad & \|\delta y_\text{perf}(t_f)\|_{Q_f}^2+\int_{t_0}^{t_f}\bigl(\|\delta y_\text{perf}(t)\|_Q^2+\|\Delta u(t)\|_{R_\Delta}^2\bigr)\,\mathrm{d}t,\\
  \text{s.t.}\quad 
    & z(t_0) = V^\top\bigl(\tilde{x}_o(t_0)\bigr),\\
    & \dot{z}(t) = r\bigl(z(t)) +  r_{ref}(t),\\
    & y_\text{perf}(t) = C\,w\bigl(z(t)\bigr) ,\\
    & y_\text{perf}(t)\in\mathcal{Y},\quad u(t)\in\mathcal{U} ,
\end{aligned}
\label{eq:low_dim_mpc}
\end{equation}
where $y_\text{perf} \in \mathbb{R}^p$ is the performance state, which is a subset of the observed states, i.e. $p<o$. The cost penalizes deviations of the performance trajectory $\delta y_\text{perf}(t) := y_\text{perf}(t) - y_\text{ref}(t)$ as well as variations in the control input $\Delta u(t) := u(t) - u(t- dt)$.
This formulation emphasizes computational tractability by restricting the optimization to the reduced coordinates while still enforcing feasibility in the full space.

\begin{figure*}[tbh]
    \centering
    \vspace{0.9em}
    \includegraphics[width=0.77\textwidth]{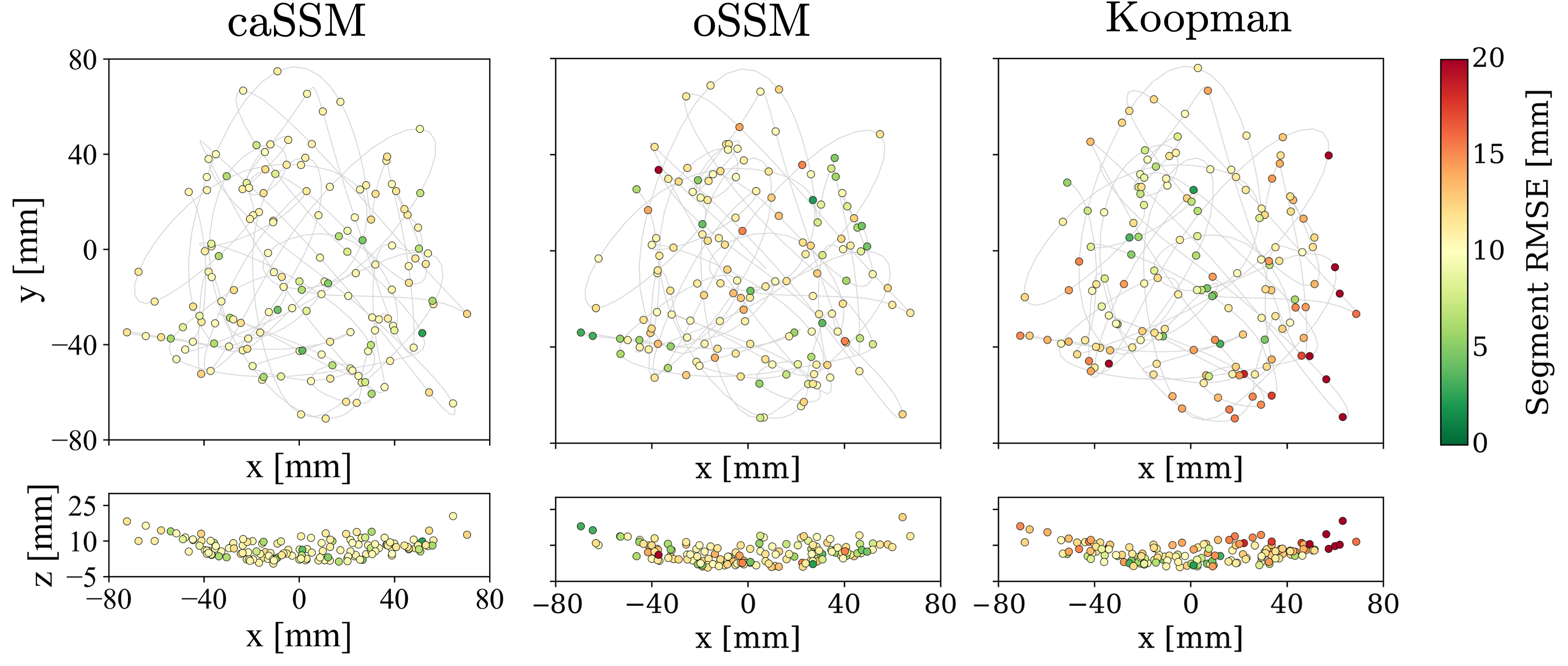}
    \vspace{-1.0em}
    \caption{RMSE of predictive results across three models for the same ground truth trajectory. caSSM shows lowest RMSE across all segments, while Koopman diverges near training data boundaries.}
    \label{fig:results_ol_predictions}
    \vspace{-1.0em}
\end{figure*}

\begin{figure}[b!]
    \centering
    \includegraphics[width=0.7\linewidth]{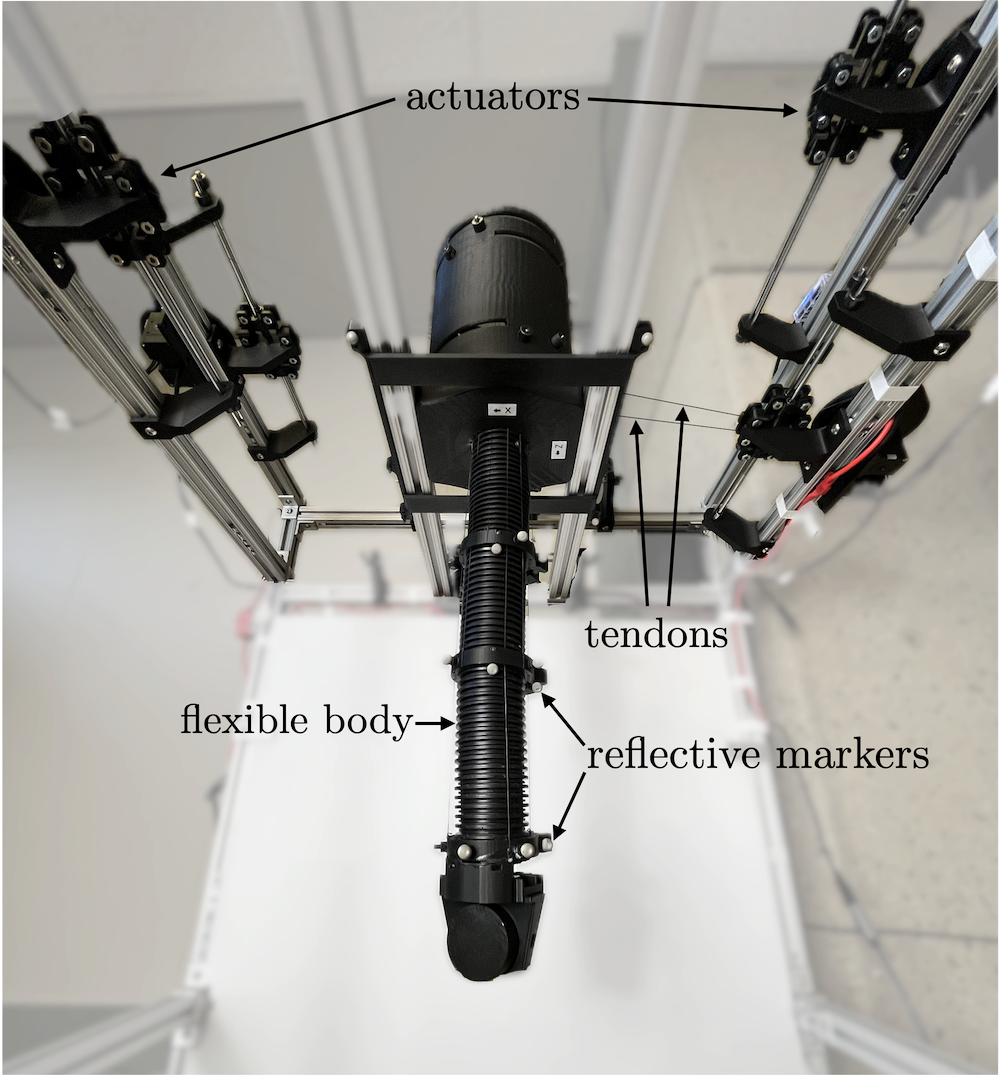}
    \caption{The tendon-driven trunk hardware system, shown in its rest position. The flexible body has a length of \qty{31}{\centi\meter}. Tendons extend from the actuators to rigid clamps attached to the flexible body. Observations were captured using an OptiTrack motion capture system.}
    \label{fig:trunk_hardware}
\end{figure}

\section{Experiments}
\label{sec:experiments}

This section presents our experimental evaluation.
We first detail the experimental setup and then demonstrate caSSM effectiveness in two experiments: open-loop predictive performance and closed-loop reference trajectory tracking. In both settings we benchmark against two baselines.

\subsection{Experimental Setup}
\label{sec:experimental_setup}
Experiments are performed on a tendon-driven continuum ``trunk” robot actuated by six servos in antagonistic tendon pairs, grouped into two independent inputs ($u \in \mathbb{R}^2$), with 3D position measurements coming from motion capture tracking of reflective markers placed on three segments along the soft body ($x_o \in \mathbb{R}^9$). We do not measure additional tendon-internal states explicitly, as no slow tendon-related modes are observed in the measured response. Residual tendon effects are therefore captured implicitly through the nonlinear coupling term $g(x,u)$.
Relevant nonidealities include actuation dead bands near the origin position, trial-to-trial variability (uncertainty bands), and hysteresis (path dependency).
The system is shown in Fig.~\ref{fig:trunk_hardware}.

To evaluate the proposed caSSM method, we compare three reduced-order models: (i) \textbf{caSSM}, a 7D manifold including actuator modes with $\Lambda$ recovered via~\eqref{eq:Lambda_hat} and two delay embeddings; (ii) \textbf{oSSM}, a 5D polynomial manifold (degree~2) with a separately calibrated affine control map following~\cite{buurmeijer2025taminghighdimensionaldynamicslearning}; and (iii) \textbf{Koopman}, using polynomial lifting (degree~2) and one delay, yielding a 120-dimensional lifted linear model following \cite{Williams2015}. All models are trained on data sampled at $\qty{50}{\hertz}$, and their hyperparameters are tuned for best performance on a shared validation set.\footnote{Baseline results for oSSM models with RFF features are omitted, as their performance within the training domain was comparable to that of polynomial features.}

\subsection{Open-Loop Predictive Performance}\label{sec:open_loop_model_predictions}

We assess the models' predictive capabilities independent of controller feedback using a finite-horizon \emph{random-actuation} protocol.
A \(\qty{15}{\second}\) random-input sequence is partitioned into 150 segments of 5 timesteps each (\(\qty{0.1}{\second}\)).
For each segment, models are initialized at the measured start state and simulated in open loop with RK4 at \(\Delta t=\qty{0.02}{\second}\).
We compute the RMSE between predicted and measured states, report the distribution across segments in Fig.~\ref{fig:results_ol_predictions}, and summarize the segment-averaged RMSE values in Table~\ref{tab:ol_rmse_models}.

The caSSM model achieved the lowest average RMSE of \(\qty{2.8}{\milli\meter}\) across random segments, maintaining accuracy under fast transients and varied initializations. For large-amplitude circular trajectories, accuracy decreases—likely due to hysteresis not captured by decay-based training—yet the model remains more numerically stable than the baselines, aided by the chosen feature map.

\begin{table}[htbp]
  \caption{Average RMSE [mm] for open-loop predictions.}
  \label{tab:ol_rmse_models}
  \vspace{-1.0em}
  \centering
  \begin{tabular}{ l ccc}
    \toprule
    Model & caSSM & oSSM & Koopman \\
    \midrule
    RMSE & \bfseries\num[detect-weight=true]{2.8} & \num{4.6} & \num{7.3} \\
    \bottomrule
  \end{tabular}
\end{table}

The oSSM model showed moderate performance (average RMSE of \(\qty{4.6}{\milli\meter}\)), with accuracy deteriorating primarily during rapid state transitions. For trajectories with significant hysteresis effects (particularly large-amplitude circular motions), oSSM occasionally outperformed caSSM, likely because its calibration step partially compensates for hysteresis phenomena not explicitly modeled.

The Koopman baseline exhibited the highest average RMSE of \(\qty{7.3}{\milli\meter}\) and the least consistent performance.
Its predictions degraded near the boundaries of the training data, with numerous segments showing numerical instability. 
This suggests that the Koopman model overfits to measurement noise rather than capturing the underlying dynamics.
\vspace{-0.25em}
\begin{highlightbox}
\centering \textbf{Key Takeaways}
\begin{enumerate}[label=(\roman*), leftmargin=*]
  \item Embedding actuator dynamics (caSSM) significantly improves both prediction quality and numerical stability on hardware.
  \item The decay-only training introduces a modeling bias for hysteresis effects; caSSM is most robust, while oSSM trades modeling bias for reduced robustness.
  \item Koopman lifting exhibits reliable performance in-distribution and for small-amplitude motions, but fails to generalize robustly to larger amplitudes.
\end{enumerate}
\end{highlightbox}

\subsection{Closed-Loop Tracking on Hardware}
\label{sec:cl_on_hardware}
We compare caSSM against oSSM and Koopman in a closed loop on the trunk robot. All controllers use the MPC formulation in~\eqref{eq:low_dim_mpc} and performance is reported as RMSE of end-effector tracking.

Both SSM models achieved the best results using the same MPC parameters: horizon \(N{=}\num{10}\) at \(\qty{50}{\hertz}\) model rate (actuation \(\qty{25}{\hertz}\)), state costs \(Q=\mathrm{diag}(\num{7},\num{7},\num{0})\), terminal costs \(Q_N=\mathrm{diag}(\num{20},\num{20},\num{0})\), input-change penalty \(R_\Delta=\mathrm{diag}(\num{0.16},\num{0.16})\), and no absolute input penalty. State and input constraints are imposed uniformly. The Koopman model with lifted dimension \(n{=}\num{120}\) achieved its best results in closed loop by using a shorter horizon (\(N{=}\num{2}\)), which maintained predictive stability while meeting real-time solve-time constraints. Its tuned costs are \(Q=\mathrm{diag}(\num{1},\num{1},\num{0})\), \(Q_N=\mathrm{diag}(\num{20},\num{20},\num{0})\), and \(R_{\Delta u}=\mathrm{diag}(\num{0.05},\num{0.05})\).

\begin{figure}[bt]
    \centering
    \vspace{0.5em}
    \includegraphics[width=0.445\textwidth]{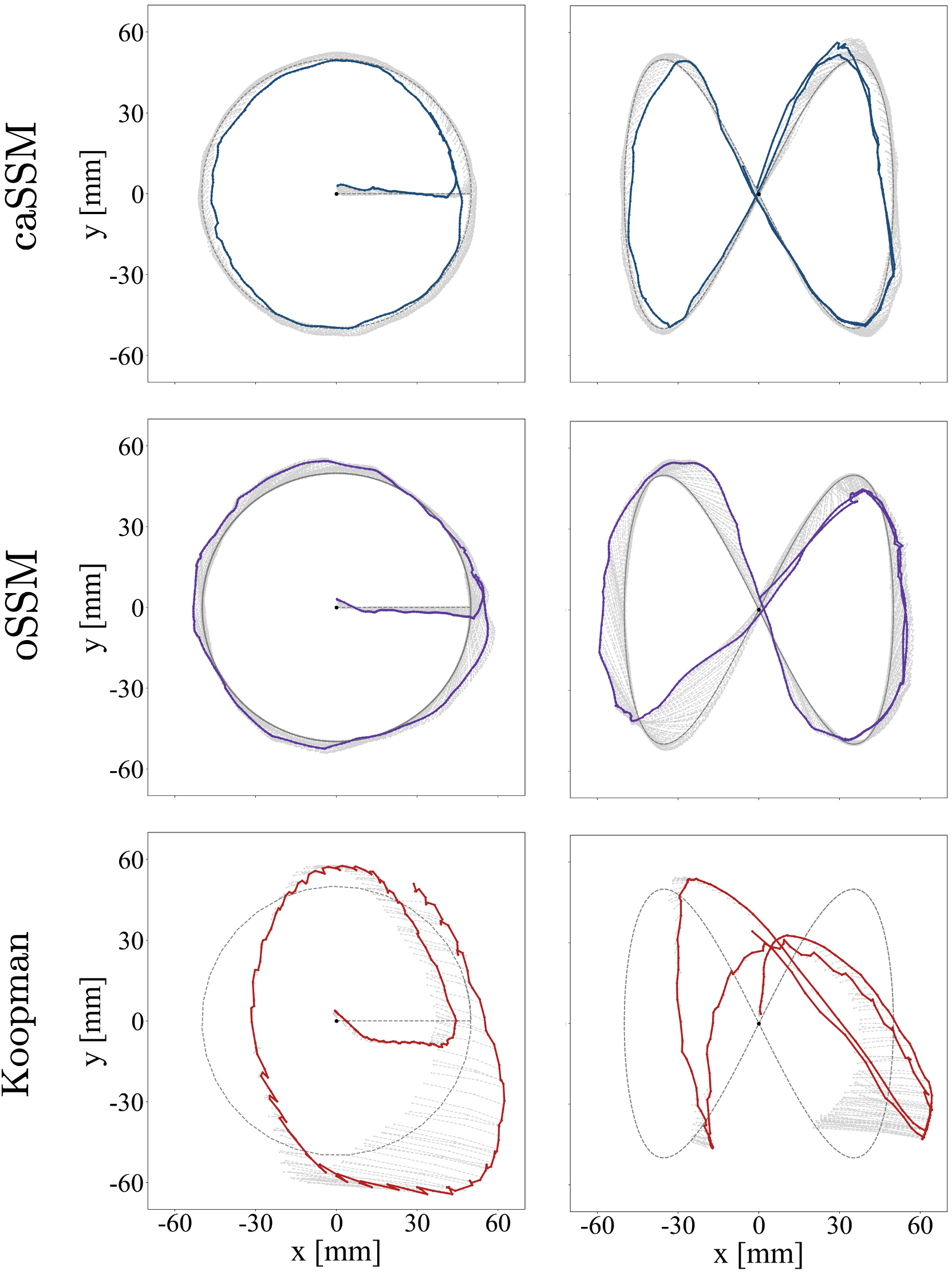}
    \vspace{-1em}
    \caption{Closed-loop MPC performance on circular and figure-eight trajectories.
Reference paths are shown in gray, with MPC rollouts in lighter gray. The Koopman model fails to track the reference, oSSM struggles on sharp curves while only caSSM tracks consistently across both shapes.}
    \label{fig:results_cl_predictions}
    \vspace{-1.5em}
\end{figure}

We compare tracking performance on circles and figure-eights of different radii at nominal (\(0.5\,\mathrm{rad/s}\)) and doubled (\(1\,\mathrm{rad/s}\), denoted "fast") angular velocity. Qualitative results for selected trajectories are shown in Fig.~\ref{fig:results_cl_predictions}, and RMSE values are summarized in Table~\ref{tab:rmse_models_cl}.

\begin{table}[b!]
  \caption{Average RMSE [mm] for closed-loop trajectories.}
  \label{tab:rmse_models_cl}
  \vspace{-1.0em}
  \centering
  \begin{tabular}{l ccc}
    \toprule
    Model & caSSM & oSSM & Koopman \\
    \midrule
    Circle \qty{5}{\centi\meter}       & \bfseries\num[detect-weight=true]{4.60} & \num{9.89} & \num{20.40} \\
    Circle \qty{5}{\centi\meter} – "fast" & \bfseries\num[detect-weight=true]{5.30} & \num{13.44} & \textit{Diverged} \\
    Circle \qty{8}{\centi\meter}        & \bfseries\num[detect-weight=true]{9.13} & \num{19.80} & \textit{Diverged} \\
    Circle \qty{10}{\centi\meter}       & \bfseries\num[detect-weight=true]{13.36} & \textit{Diverged} & \textit{Diverged} \\
    Figure-eight              & \bfseries\num[detect-weight=true]{4.40} & \num{8.64} & \num{26.72} \\
    Figure-eight – "fast"       & \bfseries\num[detect-weight=true]{8.12} & \num{15.10} & \textit{Diverged} \\
    \bottomrule
  \end{tabular}
  \vspace{-0.5em}
\end{table}

Among the compared models, caSSM achieves the highest accuracy and robustness, reaching \(4.6\,\mathrm{mm}\) RMSE on a \(5\,\mathrm{cm}\) circle (vs.\ \(9.89\,\mathrm{mm}\) for oSSM and \(20.40\,\mathrm{mm}\) for Koopman) and remaining stable at larger radii and higher speeds. At higher amplitudes, a small steady offset appears, consistent with hysteresis effects discussed in \Cref{sec:experimental_setup}, yet tracking remains symmetric and reliable. In contrast, oSSM tracks smaller and slower trajectories but degrades as amplitude and speed increase.
While the separate calibration step in oSSM can partly mitigate hysteresis effects, its overall robustness remains below that of caSSM.

Despite solving a \textit{linear} optimal control problem, the Koopman baseline has the longest average MPC solve time of \(64\,\mathrm{ms}\) and the fastest prediction divergence, which together require shorter horizons for robustness and real-time feasibility.
As a result, it can follow only small motions and yields the highest RMSE among successful runs. In contrast, caSSM’s added modeling capacity increases solve time only moderately, from \(14\,\mathrm{ms}\) for oSSM to \(18\,\mathrm{ms}\), while significantly improving tracking performance.


\vspace{-0.25em}
\begin{highlightbox}
\centering \textbf{Key Takeaways}
\begin{enumerate}[label=(\roman*), leftmargin=*]
\item Explicit actuator dynamics in caSSM expand the controllable workspace, improve tracking accuracy. 
\item caSSM enables more stable execution across diverse trajectories by providing more robust predictions.
\item Compared to Koopman, SSM-based models sustain real-time MPC with longer horizons, yielding improved closed-loop stability.
\end{enumerate}
\end{highlightbox}

\section{Conclusion and Outlook}\label{sec:conclusion_outlook}

By explicitly embedding actuator dynamics, our caSSMs close key modeling gaps in continuum robots and enable real-time control with minimal manual calibration.

Compared to regular SSM and Koopman baselines, caSSM delivers higher-fidelity open-loop predictions and superior closed-loop tracking with real-time MPC on a continuum robot hardware platform.
Practically, the pipeline streamlines data collection to a single automated step (decay trajectories), removes separate control calibration, and maintains tractable solve times via a simplified reference-injection scheme.
Together, these elements make caSSMs a practical and scalable tool for modeling and control of continuum robots.

Our approach exhibits two notable limitations. First, training exclusively on controlled decay trajectories can introduce residual hysteresis bias, which manifests in our system as small steady-state offsets in predictions of higher-amplitude trajectories.
Despite these offsets, the closed-loop system under MPC maintains both stability and symmetry properties.
Second, the method's effectiveness depends on accurate actuator mode identification; improperly identified actuator dynamics can degrade closed-loop performance beyond what would be observed without explicit actuator modeling.

This work motivates several extensions that we plan to investigate in future studies: (i) \textbf{Adiabatic extensions:} interpolating local models around forced steady states (as done in \cite{Kaundinya2025aSSM}) may improve setpoint projection in regions where global decays underrepresent effects like hysteresis. 
(ii) \textbf{Design-space tuning:} principled selection of actuator spectra \(\Lambda\) (and its excitation profile) to balance identifiability, numerical conditioning, and control authority. 
(iii) \textbf{Controller-side coupling:} structured design of \(H\) and \(h(x,u)\) for bandwidth-aware stabilization of unstable modes, including actuation-awareness.

\section*{Acknowledgments}
This work was partially supported by Toyota Motor Engineering \& Manufacturing North America (TEMA).
The views expressed in this article are solely those of the authors and do not necessarily reflect those of the supporting entity. 
We thank Patrick Benito Eberhard for valuable discussions on hardware evaluations.



\bibliographystyle{IEEEtran}
\bibliography{references}


\appendix

\subsection{Derivation of Conditions for Stabilizing Controller}\label{app:stability}

\paragraph*{Proposition}
Consider the augmented closed-loop matrix
\begin{equation}
\tilde A(H) \;:=\;
\begin{bmatrix}
A^{\mathrm{u}} & A_u \\
H & \Lambda
\end{bmatrix}.
\end{equation}
Suppose $(A^{\mathrm{u}},A_u)$ is stabilizable, and let $K$ be any matrix such that $A_c := A^{\mathrm{u}} + A_u K$ is Hurwitz (which exists by stabilizability).
Moreover, assume $\Lambda \preceq -\beta I$ for some $\beta>0$. If $\|K A_u\|_2 < \beta$, then for $H = K A^{\mathrm{u}} - \Lambda K + K A_u K$, the matrix $\tilde A(H)$ is Hurwitz.

\begin{proof}
Choose $K$ as in the statement, and consider the similarity transform
\[
T := \begin{bmatrix} I & 0 \\ -K & I \end{bmatrix}, \quad T^{-1} = \begin{bmatrix} I & 0 \\ K & I \end{bmatrix} ,
\]
which corresponds to the coordinate change $v := u - Kx$. Then
\[
\tilde A_\triangle \;:=\; T\,\tilde A(H)\,T^{-1}
\]
is similar to $\tilde A(H)$ and therefore has the same spectrum.

For the prescribed choice
\[
H := K A^{\mathrm{u}} - \Lambda K + K A_u K,
\]
a direct calculation gives
\[
\tilde A_\triangle
=
\begin{bmatrix}
A^{\mathrm{u}} + A_u K & A_u \\
0 & \Lambda - K A_u
\end{bmatrix}
=
\begin{bmatrix}
A_c & A_u \\
0 & \Lambda - K A_u
\end{bmatrix}.
\]

Hence $\tilde A_\triangle$ is block upper triangular, and therefore
\[
\sigma\!\bigl(\tilde A(H)\bigr)
=
\sigma(\tilde A_\triangle)
=
\sigma(A_c)\cup\sigma(\Lambda - K A_u).
\]
By construction, $A_c$ is Hurwitz. Moreover, since $\Lambda \preceq -\beta I$ and $\|K A_u\|_2 < \beta$, the matrix $\Lambda - K A_u$ is Hurwitz. Thus both diagonal blocks of $\tilde A_\triangle$ are Hurwitz, so $\tilde A_\triangle$ is Hurwitz. Because $\tilde A_\triangle$ is similar to $\tilde A(H)$, it follows that $\tilde A(H)$ is Hurwitz.
\end{proof}

Intuitively, the actuator must respond sufficiently faster than the feedback-induced coupling term $K A_u$. Otherwise, actuator lag weakens the effective feedback action and can prevent stabilization of the augmented system.

\subsection{Feature Maps}
\label{app:feature_maps}
\begin{figure}[t]
    \centering
     \includegraphics[width=0.85\linewidth]{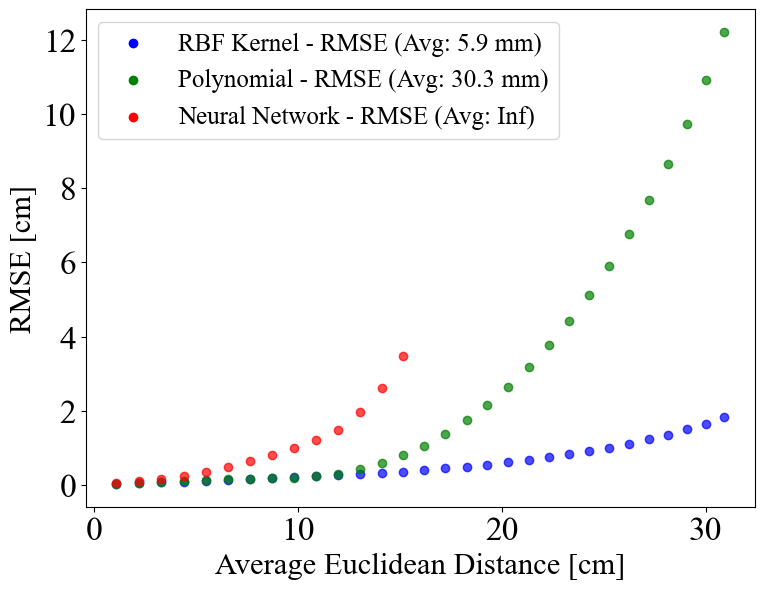}
    \caption{Test-set performance for models trained with three feature maps (RFF—blue, polynomials—green, neural networks—red). Training data covers states within \(15\,\mathrm{cm}\) of the origin.}
    \label{fig:featuremaps_results}
\end{figure}


Further comparative results for the different feature maps are shown in Fig.~\ref{fig:featuremaps_results}. Among the evaluated candidates, RFF-based feature maps proved the most numerically stable, especially in regions of unseen data. We therefore adopt RFF for all nonlinear mappings. These kernel-inspired cosine features approximate an RBF kernel yielding an implicit infinitely dimensional embedding space while avoiding the $\mathcal{O}(N^2)$ scaling of kernel-matrix methods \cite{rahimi2007random}. Concretely, $\omega$ and $b$ are sampled according to \eqref{eq:rff_sampling}, and the resulting feature map is defined in \eqref{eq:rff_features}.

\begin{equation}
\begin{aligned}
\omega_i &\sim p(\omega), \quad \text{(}\mathcal{N}(0,\ell^{-2}I)\text{ for RBF)}, &&\quad i=1,\dots,D,\\
b_i &\sim \mathrm{Uniform}[0,2\pi], &&\quad i=1,\dots,D.
\end{aligned}
\label{eq:rff_sampling}
\end{equation}

\begin{equation}
\begin{gathered}
\phi_{\mathrm{RFF}}(x) = \sqrt{\tfrac{2}{D}}\,
  \bigl[\cos(\omega_1^\top x{+}b_1),\dots,\cos(\omega_D^\top x{+}b_D)\bigr]^\top,\\[0.2em]
f(x) = \theta^\top \phi_{\mathrm{RFF}}(x).
\end{gathered}
\label{eq:rff_features}
\end{equation}
Here, $\ell$ denotes the kernel length scale, with larger values inducing smoother functions and smaller values enabling the representation of finer-scale variation. The parameter $D$ specifies the number of sampled features and governs the trade-off between approximation fidelity and computational complexity.

\noindent\textbf{Practical note:}
For real-time control, we implement models and MPC in \textit{JAX}.
While RFF are easy to code in JAX, manually sampling \(\omega,b\) with \texttt{jax.random} produced inconsistent accuracy, likely due to scaling and normalization mismatches.
Using \textit{scikit-learn}'s RFF to sample \(\omega,b\) and importing them into JAX was consistently reliable, highlighting its sensitivity to implementation details. The implementation is available at \url{https://github.com/StanfordASL/cassm}.

\end{document}